\documentclass{article}

\PassOptionsToPackage{numbers, compress}{natbib}

\usepackage[preprint]{neurips_2020}

\usepackage[utf8]{inputenc} 
\usepackage[T1]{fontenc}    
\usepackage{hyperref}       
\usepackage{url}            
\usepackage{booktabs}       
\usepackage{amsfonts}       
\usepackage{nicefrac}       
\usepackage{microtype}      
\usepackage{amsmath}
\usepackage{booktabs}
\usepackage{makecell}
\usepackage{graphicx}
\usepackage{subfigure}
\usepackage{multirow}
\usepackage{tabu}
\usepackage{pdfpages}
\usepackage{verbatim}
\usepackage[noend]{algpseudocode}
\usepackage{algorithmicx,algorithm}
\newtheorem{theorem}{Theorem}

\newtheorem{proof}{Proof}

\title{Adversarial Eigen Attack on Black-Box Models}

\author{%
  Linjun Zhou \\
  Tsinghua University\\
  \texttt{zhoulj16@mails.tsinghua.edu.cn} \\
  \And
  Peng Cui \\
  Tsinghua University\\
  \texttt{cuip@mail.tsinghua.edu.cn} \\
  \And
  Yinan Jiang \\
  China Academy of Electronics and Information Technology\\
  \texttt{jiang\_yinan@126.com} \\
  \And
  Shiqiang Yang \\
  Tsinghua University\\
  \texttt{yangshq@mail.tsinghua.edu.cn} \\
}

\begin{document}

\maketitle

\begin{abstract}
  Black-box adversarial attack has attracted a lot of research interests for its practical use in AI safety. Compared with the white-box attack, a black-box setting is more difficult for less available information related to the attacked model and the additional constraint on the query budget. A general way to improve the attack efficiency is to draw support from a pre-trained transferable white-box model. In this paper, we propose a novel setting of transferable black-box attack: attackers may use external information from a pre-trained model with available network parameters, however, different from previous studies, no additional training data is permitted to further change or tune the pre-trained model. To this end, we further propose a new algorithm, EigenBA to tackle this problem. Our method aims to explore more gradient information of the black-box model, and promote the attack efficiency, while keeping the perturbation to the original attacked image small, by leveraging the Jacobian matrix of the pre-trained white-box model. We show the optimal perturbations are closely related to the right singular vectors of the Jacobian matrix. Further experiments on ImageNet and CIFAR-10 show that even the unlearnable pre-trained white-box model could also significantly boost the efficiency of the black-box attack and our proposed method could further improve the attack efficiency.
\end{abstract}

\section{Introduction}
\vspace{-0.3em}

With the development of deep learning, machine learning systems have presented an explosive growth in the industry. Despite the convenience and automation they bring to our life, the security problem of these systems has arisen intensively. Black-box attack \cite{papernot2017practical} is one kind of attack that produces a threat to modern machine learning systems. Generally, the machine learning system providers would not actively provide the algorithmic details of their system to the users, rather only exposing the input and the corresponding output, which may effectively prevent the so-called white-box attack \cite{papernot2016distillation}. However, recent studies \cite{ilyas2018black, guo2019simple, moon2019parsimonious} show that the black-box attack, which causes misclassification of the system with just slightly perturbing the input under the hidden deployed model, still works well at a small query cost. In a sense, the adversarial attack and defense are just spears and shields, and both of them could promote the development of machine learning towards robustness.

In previous studies, there are two kinds of general settings related to the black-box attack. One is the black-box attacks with gradient estimation, and the other is the black-box attacks with substitute networks \cite{moon2019parsimonious} (also called transfer-based attack in \cite{huang2019black}). The former describes a pure black-box attack setting, where the available information is just the input and the output of the black-box model. A common technique used in this setting is the zeroth-order optimization \cite{ilyas2018black}. Different from white-box attack, there is no gradient information related to network parameters available in black-box attack. The gradient needs to be estimated by sampling different directions of perturbation and aggregating the relative changes of a certain loss function related to the output. The latter uses side information from a training dataset. Generally, a substitute white-box model is trained on the given training dataset. Hence, while processing the black-box attack, the gradient information of the white-box model could be utilized to help improve the efficiency of attack.

In this paper, we concentrate on the transfer-based attack scenario, but with a new setting different from previous work. In the practice of deep learning, in some cases, totally re-training a complex model is time-consuming or even more detrimental situation is that there is no approach to get enough training data. Under these cases, previous transfer-based attack methods may fail to work. Fortunately, we could seek help from some pre-trained models. Our assumption is that, a pre-trained white-box model is given (\textit{i.e.} its network structure and parameters), but there is no additional training dataset available. In other words, no extra module could be added or substituted to the given model and no additional training step is permitted before applying the pre-trained white-box model to the black-box attack. All available information is the consecutive representation space formed by the pre-trained network. We expect that a pre-trained white-box model with strong generalization ability could enhance the efficiency of the black-box attack.

One of the challenges in this setting is to solve the distribution shift of the conditional probability $P(y|x)$ in the white-box model and the black-box model. Specifically, given the same input $x$, the two models may show inconsistency on the output value. This will lead to disagreement on gradient direction when attacking the black-box model with the white-box model, \textit{i.e.} the steepest descent direction on decreasing $P(y|x)$ of the white-box model may not be the actual direction of the black-box model. In this paper, our solution is to combine the white-box attack and the black-box attack. By viewing the mapping from the intermediate representation of the white-box model to the output of the black-box model as a black-box function, a substitute black-box attack setting on the representation space is formed, and common practices of black-box attack could be applied. On the other side, the mapping from the original input to the intermediate representation layer is a part of the pre-trained model, which could be seen as a white-box setting. It is also noteworthy that the framework can deal with the same or different classification categories of two models, enhancing its practical application scenarios.  The main reason using the representation space of a pre-trained white-box network could help promote the attack efficiency of a black-box model is that, the lower layers of the deep neural network, \textit{i.e.} the representation learning layers, are transferrable across different datasets or data distributions \cite{yosinski2014transferable}.

Yet another goal for an efficient black-box attack algorithm is to decrease the query number to the black-box model, while keeping the perturbation to the original input sample as small as possible. To balance the two requirements, we propose a novel Eigen Black-box Attack (EigenBA) method. We combine the gradient-based white-box method and the SimBA algorithm \cite{guo2019simple} for the black-box part, as previously described. We further prove that the most efficient way of processing transferred black-box attack is just to process singular value decomposition to the Jacobian matrix of the representation layer of the white-box model to its input, and to perturb the input sample with the right singular vectors corresponding to the $k$ largest singular values iteratively. We formulate the optimization problem and conduct analysis in Section 3.

Our algorithm has a close tie with the SimBA algorithm \cite{guo2019simple}. They show when processing a pure black-box attack without additional information, a discrete cosine transform (DCT) basis is particularly query-efficient. In this paper, we further show that when there is a pre-trained white-box model as a prior, the most query-efficient way is based on the right singular vectors of processing SVD to the Jacobian matrix of a part of the white-box model. We also show that if the white-box network and the black-box network are close enough, using the intermediate representation of the white-box model may be more efficient than unsupervised method based on DCT. We will show related results in Section 4.

\vspace{-0.5em}
\section{Related Works}
\vspace{-0.3em}
\paragraph{White-Box Attack} White-box attack requires knowing all the information of the attacked model. As the earliest research field among adversarial attacks, there has been a vast literature on the white-box attack, and we will only cover methods with first-order gradient attack in this part, which is closely related to our topic. The adversarial examples are first proposed by \cite{szegedy2013intriguing}. They found that adding some specific small perturbations to the original samples may lead to classification errors of the neural network and \cite{goodfellow2014explaining} further explains this phenomenon as the linear behavior in high-dimensional space of neural networks. Later on, several algorithms are proposed to find adversarial examples with a high success rate and efficiency. Classical first-order attack algorithms include FGSM \cite{goodfellow2014explaining}, JSMA \cite{papernot2016limitations}, C\&W attacks \cite{carlini2017towards}, PGD \cite{madry2017towards}. The common principle for these methods is to iteratively utilize the first-order gradient information of a particular loss function with respect to the input of the neural networks. Specifically, the direction of the perturbation for each iteration is determined by a certain transformation of the gradient.

\paragraph{Black-Box Attack} Black-box attack deals with the case when the attacked model is unknown, and the only way to obtain the information of the black-box model is to iteratively query the output of the model with an input. Hence, the efficiency evaluation of the black-box model includes three aspects: success rate, query numbers and the $l_2$ or $l_\infty$ norm of the perturbation to original sample. Black-box attack could be divided to two categories: \textbf{black-box attacks with gradient estimation} and \textbf{black-box attacks with substitute networks} \cite{moon2019parsimonious}. The former uses a technique called zeroth-order optimization. Typical work includes NES \cite{ilyas2018black}, Bandits-TD \cite{ilyas2018prior}, LF-BA \cite{guo2018low}, SimBA \cite{guo2019simple}. The idea of these papers is to estimate gradient with sampling. More recently, some works view the problem as black-box optimization and propose several algorithms to find the optimal perturbation, for example, \cite{moon2019parsimonious} uses a submodular optimization method, \cite{ru2020bayesopt} uses a bayesian optimization method and \cite{meunier2019yet} uses an evolutional algorithm. The latter utilizes separate substitute networks trained to match the prediction output of the attacked network. The substitute network could be trained on additional samples. The concept is first proposed by \cite{papernot2017practical}. Typical works include AutoZOOM \cite{tu2019autozoom}, TREMBA \cite{huang2019black}, NAttack \cite{li2019nattack}, P-RGF \cite{cheng2019improving}. The efficiency of these transfer-based methods is largely depended on the quality of the substitute networks. If there is a huge distribution shift between two networks, the transfer-based method may underperform the methods with gradient estimation.

\vspace{-0.5em}
\section{Models}
\vspace{-0.5em}
\subsection{Problem Formulation}
\vspace{-0.3em}
Assume we have an input sample $x \in \mathbb{R}^n$ and a black model $F: \mathbb{R}^n \rightarrow [0,1]^{c_b}$, classifies $c_b$ classes with output probability $p_F(y|x)$ with unknown parameters. The general goal for black-box attack is to find a small perturbation $\delta$ such that the prediction $F(x + \delta) \neq y$. A common practice for score-based black-box attack is to iteratively query the output probability vector given an input adding an evolutional perturbation. Three indicators are used to reflect the efficiency of the attack algorithm: the average query number for attacking one sample, the success rate and average $l_2$-norm or $l_\infty$-norm of the perturbation (\textit{i.e.} $||\delta||_2$ or $||\delta||_\infty$).

We propose a novel setting of transfer-based black-box attack. We further assume there is a white-box model $G(x) = g \circ h(x)$, where $h: \mathbb{R}^n \rightarrow \mathbb{R}^m$ maps the original input to a low-dimensional representation space, and $g: \mathbb{R}^m \rightarrow [0,1]^{c_w}$ maps the representation space to output classification probabilities, $c_w$ is the number of classes with respect to $G$. The original classes for classifier $F$ and $G$ may or may not be the same. The parameters of $g$ and $h$ are known, but are not permitted to be further tuned by additional training samples. Our goal is to utilize $G$ to enhance the efficiency of attacking the black-model $F$ given an input $x$. \textit{i.e.} to decrease the query number for black-box model under the same level of perturbation norm.

\vspace{-0.5em}
\subsection{The EigenBA Algorithm}
\vspace{-0.3em}
One of the main challenges is that the white-box pre-trained model $G$ may show a distribution shift to the actual attacked model $F$. Even with the same output classes, the probability $p_G(y|x)$ may be different from $p_F(y|x)$. Hence, directly attacking $p_G(y|x)$ based on white-box methods may not work well on $F$, not to mention a different output classes case. Our solution is to combine the white-box attack and the black-box attack. Specifically, we do not use all the parameters of the white-box model. Instead, we utilize the intermediate representation $z = h(x)$, which is viewed as a white-box module, and a new mapping function, naming $\tilde{g}: \mathbb{R}^m \rightarrow [0,1]^{c_b}$, maps the representation space to the output of the attacked model $F$. The function $\tilde{g}$ could be seen as a new black-box target. If there indeed exists such a function, we immediately have $F = \tilde{g} \circ h$. And we will keep it as a hypothesis in the following analysis.

Hence, the black-box attack could be seen as a new optimization problem:
\begin{equation}
\label{eqn:problem}
    \min_{\delta} p_F(y|x+\delta) \Rightarrow \min_{\delta} p_{\tilde{g} \circ h}(y | x+\delta) \qquad s.t. \qquad ||\delta||_2 < \rho
\end{equation}
Here in this paper, we only consider the $l_2$-attack. Using a gradient-descent method to iteratively find an optimal perturbation is given by $x_{t+1} = x_t - \epsilon \cdot \nabla_x [F(x; \theta)_y]$. As $\nabla_x [F(x; \theta)_y]$ is unknown in black-box model, we need to estimate it by sampling some perturbations and aggregating the relative change of the output. A measure of attack efficiency is the number of samples used under the same $dp/||\delta||_2$ for each iteration, where $dp = |p_F(y|x + \delta) - p_F(y | x)|$. Specifically, the gradient could be decomposed as:
\begin{equation}
\label{eqn:F}
    \nabla_x [F(x; \theta)_y] = J_h(x)^T \nabla_z[\tilde{g}(z; \tilde{\theta})_y]
\end{equation}
where $J_h(x)$ is the $m \times n$ Jacobian matrix $\frac{\partial(z_1, z_2, \cdots, z_m)}{\partial(x_1, x_2, \cdots, x_n)}$ with respect to $h$, and the subscript $y$ represents the $y$-th component of the output of $\tilde{g}$. As $h$ is a white-box function, we could obtain the exact value of $J_h(x)$. In contrast, $\tilde{g}$ is a black-box function, we need to estimate the gradient $\nabla_z[\tilde{g}(z; \tilde{\theta})]_y$ by sampling. As the equation below holds given by the definition of directional derivatives:
\begin{equation}
\label{eqn:g}
    \nabla_z[\tilde{g}(z; \tilde{\theta})_y] = \sum_{i=1}^m \left( \left. \frac{\partial \tilde{g}(z; \tilde{\theta})_y}{\partial \vec{l_i}} \right|_z \cdot \vec{l_i} \right), \quad \vec{l_1}, \vec{l_2}, \cdots, \vec{l_m} \; are \; orthogonal.
\end{equation}
To completely recover the gradient of $\tilde{g}$, we could iteratively set the direction of the perturbations of $z$ from a group of orthogonal basis, which totally uses $m$ samples for each iteration.

In practice, the query operation in black-box attack is costly. In fact, we don't need so many samples to completely recover the gradient. Sacrificing some precision of gradient estimation can reduce the number of samples. Next, we will introduce our EigenBA algorithm to maximize the efficiency of the attack. The roadmap is that, first, we will introduce a greedy method to explore the basis of the representation space, and then we will prove the corresponding group of basis is the most query efficient under the limit of any number of queries. 

The idea of finding the orthogonal basis on the representation space is to greedily explore directions of perturbation on the original input space to maximize relative change of representation. Specifically, considering the first-order approximation of the change in representation space given by:
\begin{equation}
    \vec{l_i} = J_h(x) \delta_i
\end{equation}
where $\delta_i$ is the perturbation on original input space resulting the change of the representation space to be $\vec{l_i}$, the optimal perturbation could be seen as solving the following iterative problem:
\begin{equation}
    \label{eqn:opt}
    \begin{split}
    \max_{\delta_1} || J \delta_1 ||_2 \qquad & s.t. \qquad ||\delta_1||_2 \leq \epsilon \\
    \max_{\delta_i} || J \delta_i ||_2 \qquad & s.t. \qquad ||\delta_i||_2 \leq \epsilon, \; \delta_j^T J^T J \delta_i = 0 \; for \; all \; j < i, \; for \; i > 1
    \end{split}
\end{equation}
where $J_h(x)$ is simplified as $J$. We iteratively solve $\delta_1, \delta_2, \cdots, \delta_m$ of problem given by \ref{eqn:opt}. In this way, the first constraint assures that the relative $l_2$-norm change from the original space to the representation space, \textit{i.e.} $||\vec{l_i}||_2 / ||\delta_i||_2$ reaches a maximum and the second constraint assures the changes on the representation space are orthogonal.
\begin{theorem}
    The optimal solutions for problem given by \ref{eqn:opt} are that $\delta_1, \delta_2, \cdots, \delta_m$ are just the eigenvectors corresponding to the top-m eigenvalues of $J^TJ$ .
\end{theorem}
\begin{proof}
    For the first optimization problem given by \ref{eqn:opt}, as $J^TJ$ is a real symmetric matrix, considering the eigenvalue decomposition $J^TJ=U \Sigma U^T$. Hence, we have $||J \delta_1||_2^2 = \delta_1^T J^TJ \delta_1 = \delta_1^T U \Sigma U^T \delta_1$. Let $q = U^T \delta_1$, the original optimization problem could be written as:
    \[
        \max_q \; q^T \Sigma q = \sum_{k=1}^m \lambda_k q_k^2, \qquad s.t. \qquad \sum_{k=1}^m q_k^2 \leq 1
    \]
    As $\sum \lambda_k q_k^2 \leq \lambda_1 \cdot \sum{q_k^2} \leq \lambda_1$, and the condition of equality is reached when $q = [1, 0, \cdots, 0]^T$. Therefore, easy to show that the unique solution for $\delta_1$ is given by the first column of $U$. Using the similar techique, and noticing that $\delta_j^T  J^TJ \delta_i = \delta_j^T \cdot \lambda_i \delta_i = 0$ when $\delta_i$ and $\delta_j$ are two different eigenvectors of $J^TJ$. The constraint of the second recursive problems is satisfied.
\end{proof}

Hence, if we iteratively sample the perturbation to $\delta_1, \delta_2, \cdots, \delta_m$ in order, the one-step actual perturbation $\nabla_x[F(x; \theta)_y]$ could be approximated by Equation \ref{eqn:F} and Equation \ref{eqn:g}.  Further, as the tail part of the eigenvalues may be small, \textit{i.e.} the norm of perturbation for representation space may not be sensitive to the perturbation on the original input space with the corresponding eigenvector direction. To decrease the query number without sacrificing much attack efficiency, we only keep the top-K perturbations for exploration, $\delta_1, \delta_2, \cdots, \delta_K$. The eigenvectors of $J^TJ$ could be fast calculated by processing a truncated singular value decomposition (SVD) to Jacobian matrix $J$, only keeping top K components.

The following theorem guarantees that by greedily exploring the optimal perturbations given by Problem \ref{eqn:opt}, the attack efficiency will be globally optimal for any composition of K orthogonal perturbation vectors on representation space. The proof is shown in Appendix.

\begin{theorem}
    (Property of Eigen Perturbations) Assume there is no prior information about the gradient of $\tilde{g}$ (the direction of the actual gradient is uniformly distributed on the surface of an m-dimensional ball with unit radius). Given a query budget K for each iteration, the perturbations $\vec{l_1}, \vec{l_2}, \cdots, \vec{l_K}$ on representation space and the corresponding perturbations $\delta_1, \delta_2, \cdots, \delta_K$ on input space solved by Problem \ref{eqn:opt} is most efficient among any choice of exploring K orthogonal perturbation vectors on the representation space. Specifically, the final one-step gradient for $\nabla_z[\tilde{g}(z; \tilde{\theta})_y]$ is estimated by:
    \[
    \nabla_z[\tilde{g}(z; \tilde{\theta})_y] = \sum_{i=1}^K \left( \left. \frac{\partial \tilde{g}(z; \tilde{\theta})_y}{\partial \vec{l_i}} \right|_z \cdot \vec{l_i} \right)
    \]
    and the expected change of the output probability $dp_F(y|x)$ reaches the largest with the same $l_2$-norm of perturbation on input space for all cases.
\end{theorem}

\begin{algorithm}[t]
    \caption{The EigenBA Algorithm for untargeted attack}
    \hspace*{0.02in} {\bf Input:}
    Target black-box model $F$, the substitute model $G = g \circ h$, the input $x$ and its label $y$, stepsize $\alpha$, number of singular values $K$.\\
    \hspace*{0.02in} {\bf Output:}
    Perturbation on the input $\delta$.
    \begin{algorithmic}[1]
    \State Let $\delta=0$, $\mathbf{p} = p_F(y_1, y_2, \cdots, y_{c_b}|x)$, $succ=0$.

    \While{$succ=0$}
        \State Calculate Jacobian matrix w.r.t. $h$: $J = J_h(x+\delta)$.
        \State Process truncated-SVD as trunc-svd($J$, $K$) = $U, \Sigma, V^T$.
        \State Normalize each column of $V$: $q_i$ = normalize($V[:, i]$).
        \For{$i = 1 \cdots K$}
            \State $\mathbf{p}_{neg} = p_F(y_1, y_2, \cdots, y_{c_b}|clip(x + \delta - \alpha \cdot q_i))$ // $clip(\cdot)$ for validity of the input.
            \If{$\mathbf{p}_{neg, y} < \mathbf{p}_y$}
                \State $\delta = \delta - \alpha \cdot q_i$, $\mathbf{p} = \mathbf{p}_{neg}$.
            \Else
                \State $\mathbf{p}_{pos} = p_F(y_1, y_2, \cdots, y_{c_b}|clip(x + \delta + \alpha \cdot q_i))$.
                \If{$\mathbf{p}_{pos, y} < \mathbf{p}_y$}
                \State $\delta = \delta + \alpha \cdot q_i$, $\mathbf{p} = \mathbf{p}_{pos}$.
                \EndIf
            \EndIf
            \If{$\mathbf{p}_y \neq max_{y'}\mathbf{p}_{y'}$}
                \State $succ = 1$; \textbf{break}.
            \EndIf
        \EndFor
    \EndWhile
    \State \Return $\delta$
    \end{algorithmic}
\end{algorithm}

Another important improvement is under the idea of SimBA \cite{guo2019simple}. Instead of estimating the gradient by exploring a series of directional derivatives before processing one-step gradient descent, SimBA iteratively updates the perturbation by picking random orthogonal directions and either adding or subtracting to the current perturbation, depending on which operation could decrease the output probability. The main difference is that, SimBA pursues fewer queries by using a relatively fuzzy gradient estimation. SimBA does not concern about the absolute value of the directional derivatives, but only positive or negative. In such a way, the perturbations of the orthogonal basis used to explore the real gradient could also contribute to the decrease of the output probability. Inspired by SimBA, we substitute their randomly picked basis or DCT basis to our orthogonal basis $\delta_1, \delta_2, \cdots, \delta_K$ given by solving Problem \ref{eqn:opt}. The whole process for our EigenBA algorithm is shown in Algorithm 1. Considering time efficiency, for each loop, we calculate SVD once with respect to the initial state of input of this loop and process K steps directional derivatives exploration with the corresponding K eigenvectors as perturbations. The idea of SimBA significantly reduces the number of queries, as shown in \cite{guo2019simple}.

\vspace{-0.5em}
\section{Experiments}
\vspace{-0.5em}
\subsection{Setup}
\vspace{-0.3em}
Our EigenBA algorithm is evaluated from two aspects in the experiment part. We use a ResNet-18 \cite{he2016deep} trained on ImageNet as fixed white-box pre-trained model for the first experiment, and the attacked model is a ResNet-50 trained on the same training dataset of ImageNet. The attacked images are randomly sampled from the ImageNet validation set that are initially classified correctly to avoid artificial inflation of the success rate. For all baselines, we use the same group of attacked images. For the second experiment we use a ResNet-18 trained on Cifar-100 \cite{krizhevsky2009learning} as white-box model, and the attacked model is a ResNet-18 trained on Cifar-10 \cite{krizhevsky2009learning}. And the attacked images are randomly sampled from the test set of Cifar-10. The two different settings illustrate two types of distribution shift: the white-box model may be a smaller network with the same output categories as the black-box model, or a network which is not weaker than the black-box model, but has different output classes.

We also process the untargeted attack case and the targeted attack case in both settings, same as the previous literature of black-box attack. The main difference is that the targeted attack requires the model misclassifies the adversarial sample to the assigned class, while the untargeted attack just makes the model misclassified. Compared with untargeted attack, the goal for targeted attack is to increase $p_F(c|x)$ instead of decreasing $p_F(y|x)$, where $c$ is the assigned class. Hence, we only need to make a small change to Algorithm 1 by substituting $p_F(y|x)$ by $-p_F(c|x)$.

For all experiments, we limit the attack algorithm to 10,000 queries for ImageNet, and 2,000 for Cifar-10. Exceeding the query limit is considered as an unsuccessful attack. There are 1,000 images to be attacked for each setting. We evaluate our algorithm and all baselines from 4 indicators: The average query number for success samples only, the average query number for all attacked images, the success rate and the average $l_2$-norm of the perturbation for success samples.

We compare EigenBA to several baselines. Despite our $l_2$ attack setting, we also test some state-of-the-art baselines for $l_\infty$ attack, as the $l_2$ norm $||\delta||_2$ is bounded by $\sqrt{dim(\delta)} \cdot ||\delta||_\infty$ and algorithms for $l_\infty$ attack could also be adapted to $l_2$ attack. Baseline algorithms could be divided into two branches. One of the branches is the common black-box attack with no additional information, we compare several state-of-the-art algorithms including SimBA \cite{guo2019simple}, SimBA-DCT \cite{guo2019simple} and Parsimonious Black-box Attack (ParsiBA) \cite{moon2019parsimonious}. The main concern to be explained by comparing with these methods is to show exploring the representation space provided by a pre-trained model with a slight distribution shift is more efficient than the primitive input space or low-level image space (\textit{e.g.} DCT space). The other branch is some extensible first-order white-box attack methods that could be adapted to this setting. We design two baselines: Trans-FGSM and Trans-FGM. The two baselines are based on the Fast Gradient Sign Method and the Fast Gradient Method \cite{goodfellow2014explaining}. While conducting them, we use the same pre-trained white-box model as our algorithm. The two baselines iteratively run SimBA algorithm by randomly selecting from the Cartesian basis on the representation space. And the updating rule for the perturbation on input space is given by:
\[
    \mbox{Trans-FGSM:} \quad \delta_{t+1} = \delta_t \pm \alpha \cdot sign(\nabla_x J_h(x_t; e_i))
\]
\[
    \mbox{Trans-FGM:} \quad \delta_{t+1} = \delta_t \pm \alpha \cdot \frac{\nabla_x J_h(x_t; e_i)}{||\nabla_x J_h(x_t; e_i)||_2}
\]
where $e_i$ is the selected $i_{\textit{th}}$ basis and $\nabla_x J_h(x_t; e_i)$ is the gradient of the $i_{\textit{th}}$ output representation value $z_i$ with respect to the input $x_t$. By comparing these two methods, we will show afterward that exploring the eigenvector orthogonal subspace on representation space is more efficient than other subspace, which is consistent with Theorem 2. It is noteworthy that ParsiBA and Trans-FGSM are originally for $l_\infty$ attack. More details of the experimental setting is shown in Appendix.

\vspace{-0.5em}
\subsection{ImageNet Results}
\vspace{-0.3em}

\begin{table}[tp] \footnotesize
    \caption{Results for untargeted and targeted attack on ImageNet. Max queries = 10000}
    \renewcommand\tabcolsep{3.0pt}
    \renewcommand{\arraystretch}{1.1}
    \label{tab:imagenet}
    \centering
    \begin{tabular}{cc|cccc|cccc}  
    \hline
    
    \multirow{4}*{\textbf{Method}} & \multirow{4}*{\textbf{Transfer}} & \multicolumn{4}{c|}{\textbf{Untargeted}} & \multicolumn{4}{c}{\textbf{Targeted}} \\
    \cline{3-10}
   
    ~ & ~ & \makecell{\textbf{Avg.} \\ \textbf{queries} \\ \textbf{(success)}} & \makecell{\textbf{Avg.} \\ \textbf{queries} \\ \textbf{(all)}} & \makecell{\textbf{Success} \\ \textbf{Rate}} & \textbf{Avg.} $l_2$ & \makecell{\textbf{Avg.} \\ \textbf{queries} \\ \textbf{(success)}} & \makecell{\textbf{Avg.} \\ \textbf{queries} \\ \textbf{(all)}} & \makecell{\textbf{Success} \\ \textbf{Rate}} & \textbf{Avg.} $l_2$ \\
   
    \hline  
   
    SimBA & No & 1433 & 1519 & 0.990 & 3.958 & 5762 & 6719 & 0.774 & 8.424   \\     
    SimBA-DCT & No & 947 & 1056 & 0.988 & 3.083 & 4387 & 5437 & 0.813 & 6.612   \\
    ParsiBA & No & 997 & 1312 & 0.965 & 3.957 & 5075 & 6878 & 0.634 & 8.422  \\
    Trans-FGSM & Yes & 510 & 614 & 0.989 & 4.634 & 3573 & 4807 & 0.808 & 9.484 \\
    Trans-FGM & Yes & 675 & 843 & 0.982 & 3.650  & 3562 & 5867 & 0.642 & 8.200 \\
    \textbf{EigenBA (Ours)} & Yes & 383 & 518 & 0.986 & 3.622  & 2730 & 4140 & 0.806 & 7.926\\
   
    \hline
    \end{tabular}
\end{table}

We show the results of attacking ImageNet in Table \ref{tab:imagenet}. We adjust the hyper-parameter stepsize $\alpha$ for our method and all baselines to make sure the average $l_2$-norm of perturbation is close and compare average queries and success rate for simplicity. 

Comparing EigenBA to those algorithms without transferred pre-trained model, our method uses at most 49\% query numbers for untargeted attack and about 76\% for targeted attack and reaches a comparable success rate, which demonstrates that utilizing the representation space of a smaller model could attack more efficiently than the original pixel space or manually designed low-level DCT space. Moreover, some state-of-the-art methods, \textit{e.g.} SimBA-DCT, take advantage of the general properties of images and could not be generalized to other fields. In contrast, our method is applicable to any black-box attack scenario with a pre-trained model.

Comparing EigenBA to Trans-FGM, which is more suitable for $l_2$-attack than Trans-FGSM, our method use about 61\% query numbers for untargeted attack and 71\% for targeted attack. The results demonstrate that exploring the eigenvector subspace generated by solving problem given by \ref{eqn:opt} on the representation space is more efficient than the subspace generated by randomly chosen orthogonal basis, which is consistent to our theoretic reflection in Section 3. It is noteworthy that Trans-FGM performs similar or even worse to SimBA-DCT, which shows transfer-based method is not necessarily better than pure black-box attack methods, depending on whether the representation space provided by the transferred model is strong enough and the efficiency of the algorithm itself.

Figure \ref{fig:res} further shows the change of success rate with the change of query number limit for EigenBA, SimBA-DCT and Trans-FGM. We can conclude the distribution of the query number for 1000 attacked images for each attack method. Our EigenBA algorithm performs especially better when the limit of query number is relatively small, which will significantly reduce the query cost.

\vspace{-0.5em}
\subsection{Cifar-10 Results}
\vspace{-0.3em}

\begin{table}[tp] \footnotesize
    \caption{Results for untargeted and targeted attack on Cifar-10. Max queries = 2000}
    \renewcommand\tabcolsep{3.0pt}
    \renewcommand{\arraystretch}{1.1}
    \label{tab:cifar10}
    \centering
    \begin{tabular}{cc|cccc|cccc}  
    \hline
    
    \multirow{4}*{\textbf{Methods}} & \multirow{4}*{\textbf{Transfer}} & \multicolumn{4}{c|}{\textbf{Untargeted}} & \multicolumn{4}{c}{\textbf{Targeted}} \\
    \cline{3-10}
   
    ~ & ~ & \makecell{\textbf{Avg.} \\ \textbf{queries} \\ \textbf{(success)}} & \makecell{\textbf{Avg.} \\ \textbf{queries} \\ \textbf{(all)}} & \makecell{\textbf{Success} \\ \textbf{Rate}} & \textbf{Avg.} $l_2$ & \makecell{\textbf{Avg.} \\ \textbf{queries} \\ \textbf{(success)}} & \makecell{\textbf{Avg.} \\ \textbf{queries} \\ \textbf{(all)}} & \makecell{\textbf{Success} \\ \textbf{Rate}} & \textbf{Avg.} $l_2$ \\
   
    \hline  
   
    SimBA & No & 468 & 471 & 0.998 & 0.578 & 817 & 883 & 0.944 & 0.782   \\     
    SimBA-DCT & No & 437 & 440 & 0.998 & 0.575 & 772 & 830 & 0.953 & 0.777   \\
    Trans-FGSM & Yes & 111 & 115 & 0.998 & 0.638 & 305 & 310 & 0.997 & 0.918 \\
    Trans-FGM & Yes& 129 & 135 & 0.997 & 0.524  & 369 & 419 & 0.969 & 0.747 \\
    \textbf{EigenBA (Ours)} & Yes& 95 & 99 & 0.998 & 0.472  & 241 & 244 & 0.998 & 0.692\\
   
    \hline
    \end{tabular}
\end{table}

\begin{figure}   
    \begin{minipage}{0.48\linewidth}
        \centerline{\includegraphics[width=7.0cm]{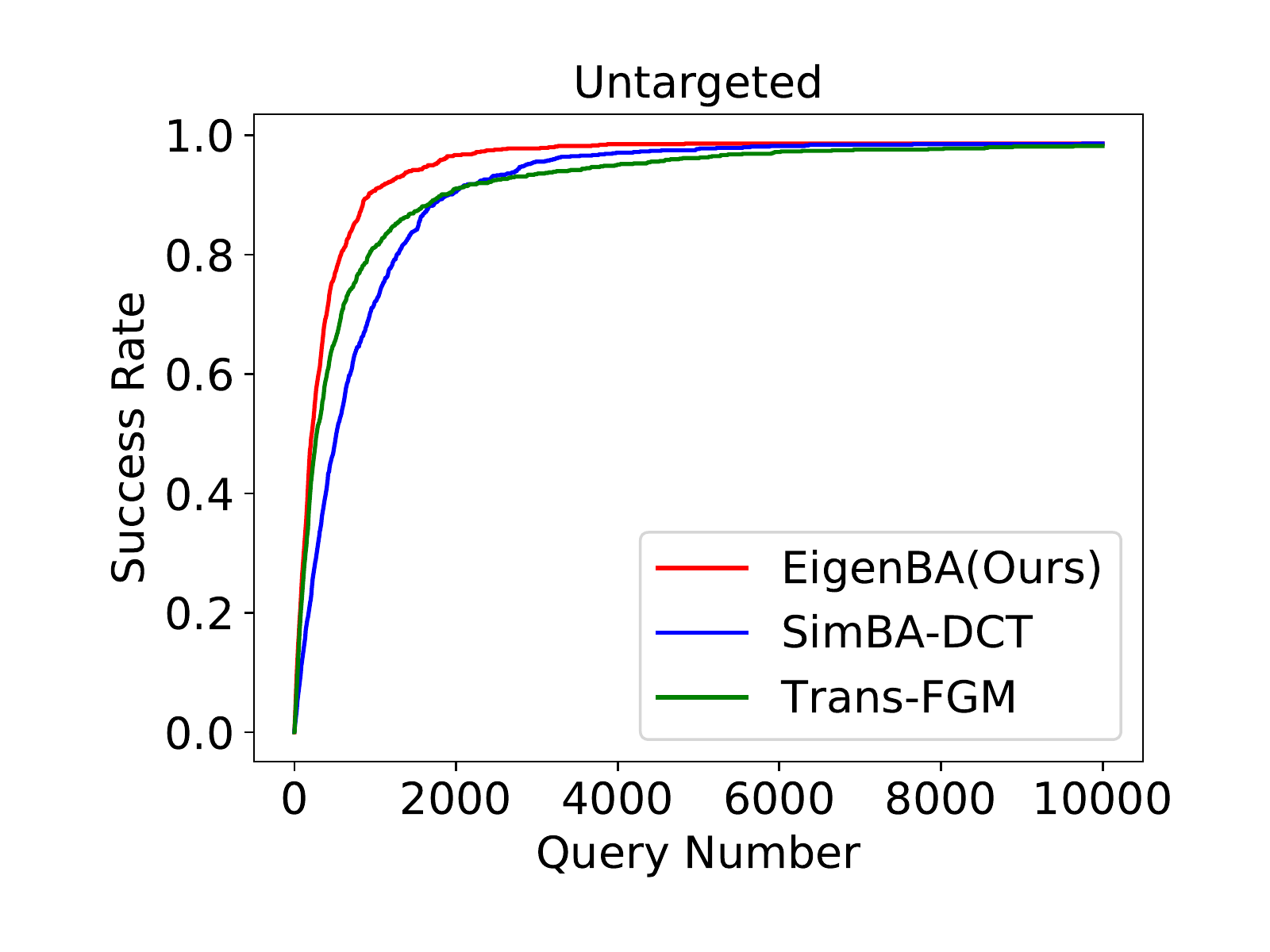}}
    \end{minipage}
    \hfill
    \begin{minipage}{.48\linewidth}  
        \centerline{\includegraphics[width=7.0cm]{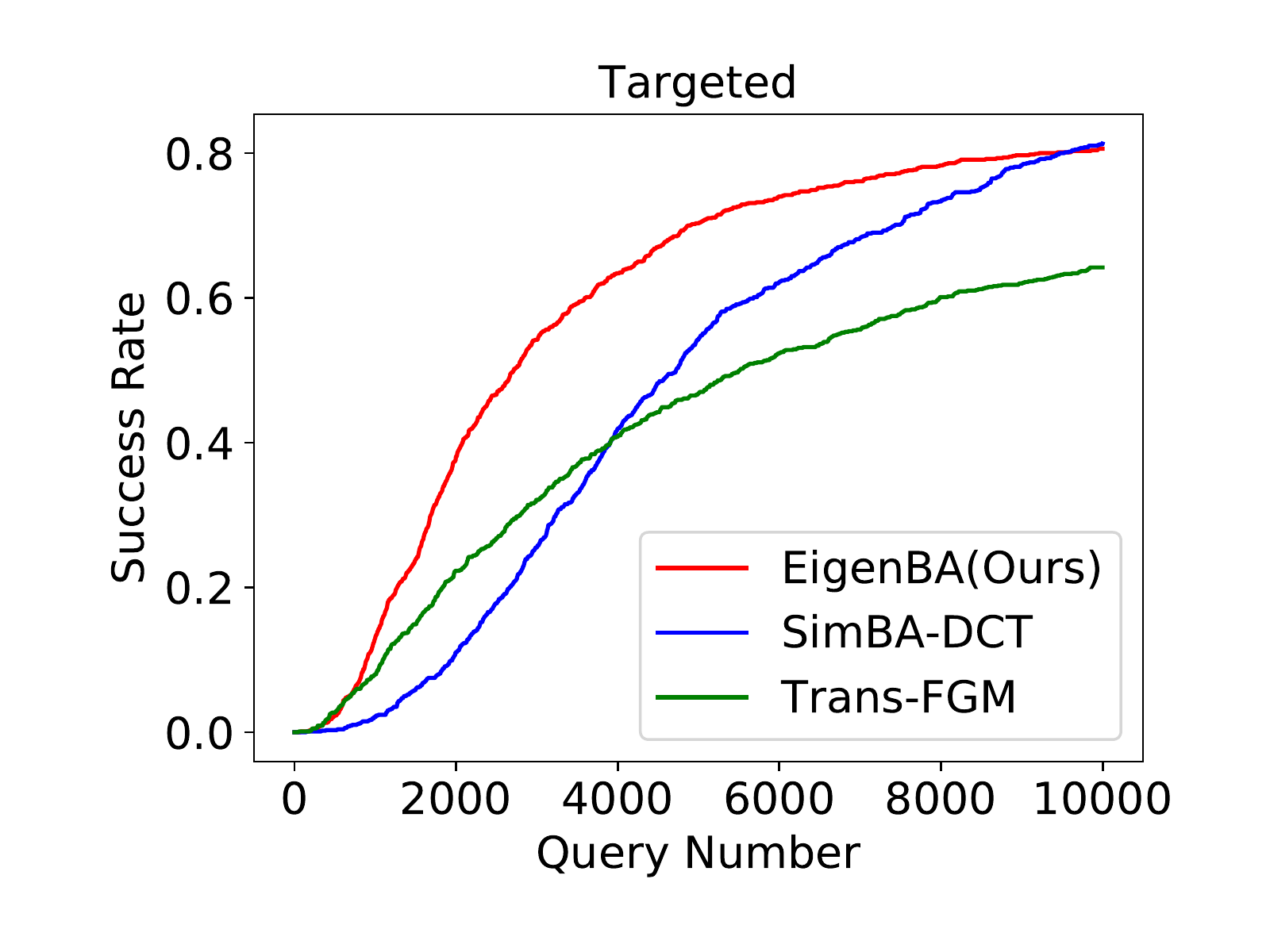}}
    \end{minipage}   
    \caption{The change of success rate with fixed query limit on ImageNet.} 
    \label{fig:res}
\end{figure}

Similar to the experiment on ImageNet, our EigenBA method still performs the best of all on attacking Cifar-10, with a transferred model trained on Cifar-100, as shown in Table \ref{tab:cifar10}. Compared with SimBA-DCT, our algorithm uses 22\% and 29\% query numbers on untargeted attack and targeted attack. Compared with Trans-FGM, the proportion is 73\% and 58\%. The results further show that even the classes of the transferred model are different from the attacked model, depending on the strong generalization ability of neural network, the representation space of the transferred network can still improve the efficiency of black-box attack.

\vspace{-0.5em}
\subsection{Ablation Study: How the generalization ability affects the efficiency of attack?}
\vspace{-0.3em}

\begin{table}[!t] \footnotesize
    \caption{Set a certain proportion of the parameters of the pre-trained model in EigenBA to zero, and compared with SimBA-DCT for attack on Cifar-10.}
    \renewcommand\tabcolsep{5.0pt}
    \renewcommand{\arraystretch}{1.1}
    \label{tab:ablation}
    \centering
    \begin{tabular}{cccccc}  
    \toprule   
   
    \textbf{Methods} & \makecell{\textbf{Parameters}\\\textbf{Reserved Rate}} & \makecell{\textbf{Avg. queries} \\ \textbf{(all)}} & \textbf{Success Rate} & \textbf{Avg.} $l_2$ & \makecell{\textbf{Pre-trained Model} \\ \textbf{Accuracy}}\\  
   
    \midrule   
   
    \multirow{6}*{EigenBA} & 1.0 & 88 & 1.000 & 0.453 & 89.19\% \\     
    ~ & 0.9 & 85 & 1.000 & 0.446 & 86.17\%   \\
    ~ & 0.8 & 130 & 0.997 & 0.459 & 77.78\% \\
    ~ & 0.7 & 195 & 0.999 & 0.560 & 69.36\%  \\
    ~ & 0.6 & 382 & 0.991 & 0.760 & 35.36\%\\
    ~ & 0.5 & 700 & 0.921 & 0.951 & 27.57\%\\
    \hline
    SimBA-DCT & - & 440 & 0.998 & 0.575 & - \\
   
    \bottomrule  
    \end{tabular}
\end{table}

From the results of Section 4.2 and 4.3, one interesting problem is how strong the generalization ability of the pre-trained white-box model can help improve the efficiency of black-box attack. In this section, we conduct an ablation study on this problem. In this experiment, we set the pre-trained model and the attacked model to be the same Resnet-18 trained on Cifar-10, but randomly setting a certain proportion of parameters to be zero for the pre-trained model. If the reserve rate of parameters is 1.0, the pre-trained model will be totally the same with the attacked model, and with the decrease of the reserve rate, the generalization ability of the pre-trained model will become weaker. Setting a random part of parameters to zero could also be seen as a change to the structure of the pre-trained network. We test the attack efficiency of EigenBA under different reserve rate ratios and compare the result with the pure black-box method SimBA-DCT in table \ref{tab:ablation}. We also report the pre-trained model accuracy in different settings by fixing network parameters below the final representation layer and only re-training the top classifier with the training dataset of Cifar-10, which reflects the generalization ability of the pre-trained model. 

The results show that when the reserve rate is larger than 0.7, the pre-trained model is helpful to the efficiency of the black-box attack (both query number and average $l_2$ are lower.). And when the reserve rate is smaller than 0.5, the model will degrade the attack efficiency. The breakeven point may appear around 0.6. It shows that even the pre-trained model cannot achieve the classification accuracy of the attacked model, it can still improve the efficiency of the black-box attack, \textit{e.g.} in this experiment, a pre-trained model with reserve rate of 0.7 just reaches 69.36\% of classification on Cifar-10, roughly equivalent to a shallow convolutional network \cite{mcdonnell2015enhanced}, which is largely below the attacked model with 89.19\%. Hence, as the representation layer of the modern neural networks generally has a strong transferability \cite{yosinski2014transferable}, our EigenBA algorithm has strong applicability in practice.

\vspace{-0.7em}
\section{Conclusions}
\vspace{-0.6em}
In this paper, we proposed a novel setting for transfer-based black-box attack. Attackers may take advantage of a fixed white-box pre-trained model without additional training data, to improve the efficiency of the black-box attack. To solve this problem, we proposed EigenBA, which iteratively adds or subtracts perturbation to the input sample such that the expected change on the representation space of the transferred model to be the direction of right singular vectors corresponding to the first $K$ singular values of the Jacobian matrix of the pre-trained model. Our experiments showed that EigenBA is more query efficient in both untargeted and targeted attack compared with state-of-the-art transfer-based and gradient estimation-based attack methods. We believe that the applicability in the real world of our algorithm will promote more research on robust deep learning and the generalization ability between deep learning models.

\section*{Broader Impact}

This paper discusses how to carry out an efficient black-box attack to a deep model, which reveals the shortcomings of robustness for deep neural networks. This paper aims to promote the study of robust learning and encourage researchers to pay more attention to the safety of machine learning algorithms. Ethical risk includes that lawbreakers may abuse our algorithm to attack mature machine learning systems, such as face recognition, intelligent driving \textit{etc.}

\bibliography{article} 
\bibliographystyle{plain}



\section*{Supplementary material}

# Appendix for Adversarial Eigen Attack on Black-Box Models

## 1 Proof of Theorem 2 in Original Paper

**Theorem 1** *(Property of Eigen Perturbations, Theorem 2 in original paper) Assume there is no prior information about the gradient of $\tilde{g}$ (the direction of the actual gradient is uniformly distributed on the surface of an m-dimensional ball with unit radius). Given a query budget K for each iteration, the perturbations $\vec{l_1}, \vec{l_2}, \cdots, \vec{l_K}$ on representation space and the corresponding perturbations $\delta_1, \delta_2, \cdots, \delta_K$ on input space solved by Problem 5 (in original paper) is most efficient among any choice of exploring K orthogonal perturbation vectors on the representation space. Specifically, the final one-step gradient for $\nabla_z[\tilde{g}(z;\tilde{\theta})_y]$ is estimated by:*

$$\nabla_z[\tilde{g}(z;\tilde{\theta})_y] = \sum_{i=1}^{K} \left( \left. \frac{\partial \tilde{g}(z;\tilde{\theta})_y}{\partial \vec{l_i}} \right|_z \cdot \vec{l_i} \right)$$

*and the expected change of the output probability $dp_F(y|x)$ reaches the largest with the same $l_2$-norm of perturbation on input space for all cases.*

**Proof 1** *Consider the relative change of $p_{\tilde{g} \circ h}(y|x)$ with respect to $x$:*

$$dy = dx^T \cdot \nabla_x[F(x;\theta)_y] = dx^T \cdot (J_h(x)^T \nabla_z[\tilde{g}(z;\tilde{\theta})_y])$$

*For simplicity, define $J = J_h(x)$, and $\tilde{g}(z;\tilde{\theta})_y = Q\beta$, where $Q = [q_1, q_2, \cdots, q_K]$, a group of orthogonal basis with unit length, and $\beta$ is $K \times 1$ vector representing K directional derivatives on Q. Thus we have:*

$$dy = dx^T J^T Q \beta$$

*To maximize $dy/||dx||_2$, we set $dx = \eta J^T Q \beta$ and then:*

$$max_{||dx||_2 \leq \epsilon} \frac{dy}{||dx||} = \frac{\eta \beta^T Q^T J J^T Q \beta}{||\eta J^T Q \beta||_2} = ||J^T Q \beta||_2$$

*As there is no prior information about the direction of the gradient of $\tilde{g}$, $\beta$ could be viewed as a random vector with a fixed length uniformly distributed on the surface of an $m$-dimensional ball. Thus, the optimization problem is converted to:*

$$max_Q \; \mathbb{E}_\beta[\beta^T Q^T J J^T Q \beta]$$

*where Q is column orthogonal.*

*Next, consider the space formed by all eigenvectors of $JJ^T$, i.e. $l_1, l_2, \cdots, l_m$, sorted by eigenvalues in descending order. Define $L = [l_1, l_2, \cdots, l_m]^T$. As Q is formed by choosing K columns from a certain orthogonal matrix, L is also an orthogonal matrix, immediately we have $Q = L^T C$ for certain $m \times K$ matrix C. Noticing that $JJ^T l_i = \lambda_i l_i$, hence,*

$$\mathbb{E}_\beta[\beta^T Q^T J J^T Q \beta] = \mathbb{E}_\beta[\beta^T C^T L J J^T L^T C \beta] = \mathbb{E}_\beta[\beta^T C^T \Sigma C \beta]$$

*where $\Sigma = diag(\lambda_1, \lambda_2, \cdots, \lambda_m)$.*

*The final step is to demonstrate $C = I_{m \times K}$ will be optimal. Consider a new random vector $\gamma = C\beta$, we have:*

$$\begin{aligned}
\mathbb{E}_\beta[\beta^T C^T \Sigma C \beta] &= \mathbb{E}_\beta[\gamma^T \Sigma \gamma] \\
&= \mathbb{E}_\beta \left[ \sum_{j=1}^{m} \lambda_j \gamma_j^2 \right] \\
&= \mathbb{E}_u \left\{ \mathbb{E}_\beta \left[ \sum_{j=1}^{m} \lambda_j \gamma_j^2 \middle| \sum_{i=1}^{K} \beta_i^2 = u \right] \right\} \\
&= \mathbb{E}_u \left\{ \sum_{j=1}^{m} \lambda_j \mathbb{E}_\beta \left[ \gamma_j^2 \middle| \sum_{i=1}^{K} \beta_i^2 = u \right] \right\} \\
&\leq \mathbb{E}_u \left\{ \sum_{j=1}^{K} \lambda_j \mathbb{E}_\beta \left[ \beta_j^2 \middle| \sum_{i=1}^{K} \beta_i^2 = u \right] \right\} \\
&= \mathbb{E}_\beta \left[ \sum_{j=1}^{K} \lambda_j \beta_j^2 \right] = \mathbb{E}_\beta[\beta^T I_{m \times K}^T \Sigma I_{m \times K} \beta]
\end{aligned} \quad (1)$$

*The key step is the inequality from 4th row to 5th row, which is not obvious. To demonstrate this, we first notice that*

$$\sum_{j=1}^{m} \mathbb{E}_\beta \left[ \gamma_j^2 \middle| \sum_{i=1}^{K} \beta_i^2 = u \right] = \sum_{j=1}^{K} \mathbb{E}_\beta \left[ \beta_j^2 \middle| \sum_{i=1}^{K} \beta_i^2 = u \right] = u \quad (2)$$

*This is because $C$ is an orthogonal transformation. Also, we have for $j = 1, 2, \cdots, K$:*

$$\mathbb{E}_\beta \left[ \gamma_j^2 \middle| \sum_{i=1}^{K} \beta_i^2 = u \right] \leq \mathbb{E}_\beta \left[ \beta_j^2 \middle| \sum_{i=1}^{K} \beta_i^2 = u \right] \quad (3)$$

*To simplify the notation, we directly use $\mathbb{E}_\beta[\gamma_j^2]$ and $\mathbb{E}_\beta[\beta_j^2]$. We prove this conclusion as follows:*

$$\begin{aligned}
\mathbb{E}_\beta[\gamma_j^2] &= \mathbb{E}_\beta[(c_j^T \beta)^2] = \mathbb{E}_\beta[((c_{j,S} + c_{j,S^\perp})^T \beta)^2] \\
&= \mathbb{E}_\beta[((c_{j,S}^T \beta)^2] = ||c_{j,S}||_2^2 \mathbb{E}_\beta[\beta_j^2] \leq \mathbb{E}_\beta[\beta_j^2]
\end{aligned}$$

*The idea is that, let S be the subspace constructed by top-K columns of L, we decompose the jth column of the orthogonal transformation C, i.e. $c_j$ as $c_{j,S} \in S$ and $c_{j,S^\perp} \in S^\perp$. $c_{j,S^\perp}$ has no contribution to the expectation, and only $c_{j,S}$ contributes to the expectation. Due to the symmetry characteristic of $\beta$ on S, we have $\mathbb{E}_\beta[((c_{j,S}^T \beta)^2] = ||c_{j,S}||_2^2 \mathbb{E}_\beta[\beta_j^2]$. And as the maximum length of $c_j$ is 1, the conclusion is proved.*

*By Equation 2 and 3, the inequality in 1 is obvious for the eigenvalues $\lambda$ is sorted in descending order.*

*We further note that $l_1, l_2, \cdots, l_K$ and the $\delta_1, \delta_2, \cdots, \delta_K$ are just the top-K eigenvectors of $JJ^T$ and $J^T J$, completing the proof.*

## 2 Implementation Details

### 2.1 Taking Advantage of Image Continuity

In practical experiments, the operation of truncated SVD is time-consuming. Consider a network $h$ with input size $n = H \times W \times C$ and the output size $m$ (*i.e.* dimension of representation), the Jacobian matrix will be $m \times n$, and the complexity of SVD operation is $O(m^2 n)$, which may be slow when the image size is large.

A simple way to decrease the complexity is to aggregate adjacent pixels together, taking advantage of image continuity. Specifically, suppose the image size is $H_0 \times W_0$, and we would like to decrease

Table 1: Untargeted attack on ImageNet, with or without rounding technique.

| Methods | Avg. queries (success) | Avg. queries (all) | Success Rate | Avg. $l_2$ |
|---|---|---|---|---|
| EigenBA (No rounding) | 383 | 518 | 0.986 | 3.622 |
| EigenBA (Rounding) | 503 | 617 | 0.988 | 3.797 |

the input to $H_1 \times W_1$. First, we define the scale to be $s_h = \lfloor H_0/H_1 \rfloor$ and $s_w = \lfloor W_0/W_1 \rfloor$. Then, while processing EigenBA, we only change the pixel in the center area $(s_h \cdot H_1) \times (s_w \cdot W_1)$ of the original image.

The forward propagation remains unchanged, where the input is still $H_1 \times W_1$. As to the backward propagation, we only need to calculate the Jacobian matrix of representation $z$ with respect to the center area $(s_h \cdot H_1) \times (s_w \cdot W_1)$. And then, we process average pooling with scale $(s_h, s_w)$ and stride $(s_h, s_w)$ for each row of Jacobian matrix $J$. (It is noteworthy that each row of $J$ is an $(s_h \cdot H_1 \cdot s_w \cdot W_1)$ vector, hence, before average pooling operation we need to restore the vector to $(s_h \cdot H_1) \times (s_w \cdot W_1)$ matrix.) Finally, the new Jacobian matrix $J'$ should be a $m \times (H_1 \times W_1)$ matrix. After processing SVD to $J'$, the right singular vector should be a $H_1 \cdot W_1$ vector, which represents for the perturbation, naming $\delta'$. Finally, the actual $\delta$ related to the center area $(s_h \cdot H_1) \times (s_w \cdot W_1)$ pixels should be obtained by applying nearest upsampling method to $\delta'$ with scale $(s_h, s_w)$.

Through this simple method, we reduce the computation of SVD to about $1/(s_h \cdot s_w)$ of the original method.

### 2.2 Rounding Technique

For all experiments, we follow the setting of SimBA and ParsiBA for fair comparison, where the input pixel could be any real number on $[0, 1]$. However, in practical use, the value of pixel is discrete in image classfication. In this section, we introduce a simple rounding technique.

The method is rather simple: after each renewal, for each pixel value $v$, we find the integer $N$, such that $N/255 \le v < (N+1)/255$. Then, we round the value $v$ to the nearer one, either $N/255$ or $(N+1)/255$. Table 1 shows the difference whether using the rounding technique.

The results show that, the rounding technique will only slightly increase average $l_2$ and average query numbers, which is acceptable in practical use.

### 2.3 Hyperparameters

In this section, we mainly describe the hyperparameters for all settings.

For attack on ImageNet, the maximum query number of each attacked image is limited to 10,000. For SimBA, the stepsize of gradient, $\epsilon$ is set to 0.2. For SimBA-DCT, the stepsize is set to 0.2, the dimensionality of 2D frequency space is set to 28. For ParsiBA, the maximum $l_\infty$ norm is set to 0.01. For Trans-FGSM , the stepsize is set to 0.4 for untargeted attack and 0.3 for targeted attack. For Trans-FGM, the step size is set to 0.4 for both cases. For EigenBA, the stepsize is set to 0.4 for both cases, we also decrease the dimension of Jacobian matrix from $512 \times (224 * 224 * 3)$ to $512 \times (112 * 112 * 3)$ by using the method described in Appendix 2.1. For processing SVD once, we extract top 100 right singular vectors.

For attack on Cifar-10, the maximum query number of each attacked image is limited to 2,000. For SimBA, the stepsize of gradient, $\epsilon$ is set to 0.04. For SimBA-DCT, the stepsize is set to 0.04, the dimensionality of 2D frequency space is set to 32 (the same to original image size, which is optimal value in Cifar-10 experiment). For Trans-FGSM the stepsize is set to 0.08, and for Trans-FGM, the stepsize is set to 0.06. For EigenBA, the stepsize is set to 0.08 for both untargeted attack and targeted attack, the dimension of Jacobian matrix is $512 \times (32 * 32 * 3)$, we do not use method in Appendix 2.1 for Cifar-10 experiment. For processing SVD once, we extract top 50 right singular vectors.

For ablation study on Cifar-10 in Section 4.4 in original paper, the stepsize for reserve rate 1.0, 0.9, 0.8, 0.7, 0.6, 0.5 experiment is 0.08, 0.07, 0.06, 0.05, 0.04, 0.03.

## 2.4 Complexity Analysis

We run all experiments on a single Ubuntu 16.04 server, with two 6-core 12-thread CPU Intel Xeon E5-2630. Totally there are 24 threads. We use a single Nvidia Tesla P40 GPU.

The bottleneck of our algorithm is the SVD operation, which is an $O(m^2 n)$ algorithm, as Section Appendix 2.1 shows. For experiment on ImageNet, we decrease the Jacobian matrix to $m = 512, n = 112 * 112 * 3 = 37,632$. And processing SVD once costs about 2 seconds.

We use a batchsize of 5. For each loop starting with an SVD operation and ending with a number of perturbation renewals, the SVD operation will take about $2 * 5 = 10$ seconds, the several steps of perturbation renewals (between 100 to 200 steps) cost about 5 seconds. For untargeted attack on ImageNet, EigenBA will cost about 16 hours to finish 1,000 attacked images, and for targeted attack on ImageNet, the execution time will be about 100 hours for 1,000 attacked images.

## 3 Visualization of the Results

We randomly choose 4 attacked images from all attacked images. Figure 1 visualizes the original images and the adversarial images generated by our EigenBA. From Figure 1, the difference between the original image and the adversarial image is barely visible to the naked eye. However, the adversarial image is incorrectly classified by the deep neural network, which demonstrates the advantages of our algorithm. One interesting finding is that, in the setting of untargeted attack, the misclassified label tends to be similar to the original label in semantic meaning, for example, although the staffordshire bullterrier is different with the bull mastiff, they are both under the meta class of Dog.

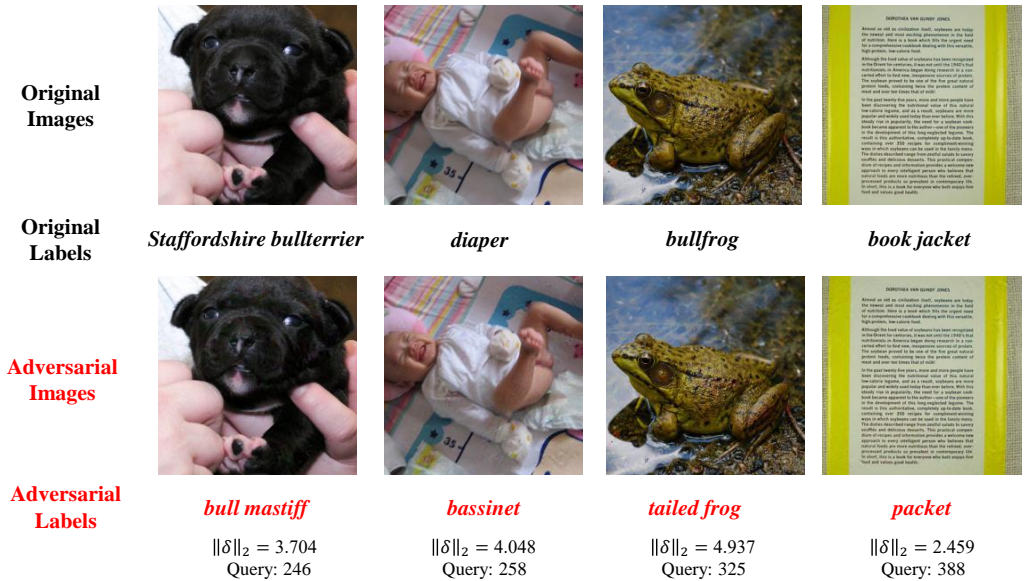

Figure 1: Showcases of untargeted attack on ImageNet.

4



\end{document}